\documentclass{article}
\pdfoutput=1
\usepackage{graphicx} 
\usepackage{dcolumn}
\usepackage{bm}
\usepackage{xcolor}
\usepackage{physics}
\usepackage{float}
\usepackage{amsfonts}
\usepackage{amsmath}
\usepackage{amsthm}
\usepackage{mathtools}
\usepackage{tcolorbox}
\usepackage{amssymb}
\usepackage{fullpage}
\usepackage{hyperref}
\usepackage{algorithm}
\usepackage{enumitem}
\usepackage[noend]{algpseudocode}
\newcommand{\ignore}[1]{}

\definecolor{scmcolor}{RGB}{0,0,255}

\definecolor{scmcolorRed}{RGB}{255,0,0}

\newtheorem{theorem}{Theorem}

\newtheorem*{theorem*}{Theorem}
\newtheorem{definition}[theorem]{Definition}
\newtheorem{corollary}[theorem]{Corollary}
\newtheorem*{corollary*}{Corollary}

\newtheorem*{conjecture*}{Conjecture}

\newtheorem{lemma}[theorem]{Lemma}

\title{Debate is efficient with your time
}

\author{
Jonah Brown-Cohen \\
\small Google DeepMind
\and
Geoffrey Irving \\
\small UK AI Security Institute
\and
Simon C. Marshall\thanks{simon.marshall@dsit.gov.uk} \\
\small UK AI Security Institute
\and
Ilan Newman \\
\small University of Haifa
\and
Georgios Piliouras \\
\small Google DeepMind
\and
Mario Szegedy \\
\small Rutgers University
}
\date{}

\usepackage[backend=biber,style=numeric,sorting=none]{biblatex}
\begin{document}
\maketitle

\begin{abstract}
AI safety via debate uses two competing models to help a human judge verify complex computational tasks. Previous work has established what problems debate can solve in principle, but has not analysed the practical cost of human oversight: how many queries must the judge make to the debate transcript? We introduce \emph{Debate Query Complexity} (DQC), the minimum number of bits a verifier must inspect to correctly decide a debate.

Surprisingly, we find that $\mathsf{PSPACE/poly}$ (the class of problems which debate can efficiently decide) is precisely the class of functions decidable with $O(\log n)$ queries. This characterisation shows that debate is remarkably query-efficient: even for highly complex problems, logarithmic oversight suffices. We also establish that functions depending on all their input bits require $\Omega(\log n)$ queries, and that any function computable by a circuit of size $s$ satisfies $DQC(f) \leq \log(s) + 3$. Interestingly, this last result implies that proving DQC lower bounds of $\log(n) + 6$ for languages in $\mathsf{P}$ would yield new circuit lower bounds, connecting debate query complexity to central questions in circuit complexity.
\end{abstract}

\section{Introduction}

Modern AI systems are trained on and produce vast quantities of data, making effective oversight and alignment challenging. As these systems grow more capable, we face a fundamental bottleneck: we must leverage comparatively small amounts of human feedback to supervise increasingly complex models. Concurrently, even expert humans will struggle to directly evaluate the outputs of advanced AI systems, meaning that our oversight system must also allow us to scale our abilities.

A promising approach to this challenge is \emph{AI safety via debate} \cite{irving2018ai, barnes2020debate}, where two AI models compete by presenting arguments to a human judge who decides which side is correct. Debate is compelling because it potentially amplifies human judgment, the judge need not solve the problem directly, only adjudicate between competing solutions. The theoretical appeal of debate has grown as formal guarantees have strengthened: with unbounded debaters and polynomial-time judges, debate can decide languages in $\mathsf{PSPACE}$ \cite{irving2018ai}. Recent work has extended similar guarantees to more realistic settings with (for example) computationally bounded debaters \cite{brown2023scalable, brown2025avoiding}.

However, previous work has not analysed the \emph{cost of human oversight}---the number of queries a judge must make to the debate transcript. This cost is critical for practical deployment, as human time is typically one of the most expensive resources. We introduce \textbf{Debate Query Complexity (DQC)}, the minimum number of bits a verifier must inspect from the debate transcript to correctly decide a Boolean function $f:\{0,1\}^n \to \{0,1\}$.

We establish tight bounds on DQC. Functions that depend on all bits of the input require $\Omega(\log n)$ queries, while any function computable by a fan-in two, boolean circuit $C$ satisfies $DQC(f) \leq O(\text{depth}(C))$ and $DQC(f) \leq O(\log(\text{size}(C)))$. These bounds employ a modification of Karchmer-Wigderson games, \emph{cross-examination} \cite{barnes2020debate,brown2023scalable}, where one debater performs a computation and the other verifies a single step.

Our main result shows that debate is remarkably query-efficient:

\begin{theorem*}
$$\mathsf{PSPACE/poly} = \{f : DQC(f) \leq O(\log n)\}$$
\end{theorem*}

Efficient debate (i.e.\ polynomial length debates and polynomial time verifiers) is equal to $\mathsf{PSPACE}$ \cite{irving2018ai}, therefore this result implies that if a debate can decide a problem, it can do so with only $O(\log n)$ steps, furthermore the verifier is itself very computationally efficient. The proof uses recursive cross-examination: Alice simulates the verifier's computation step-by-step, Bob points to an alleged error, and the verifier checks only that single location.
This suggests that effective oversight is achievable even for complex tasks, requiring only minimal human judgments. 

Our circuit upper bounds also reveal an intriguing connection to circuit complexity: proving DQC lower bounds beyond $\log n + 6$ for any language in $\mathsf{P}$ would yield circuit lower bounds exceeding the best known $\approx 5n$ bound (Corollary~\ref{cor:circuit-lower-bounds}). This implies that tight lower bounds for DQC may be as difficult to prove as circuit lower bounds. In turn, it reduces the task of proving circuit lower-bounds to a pure information theoretic setting, similar in spirit to Karchmer-Wigderson game that reduced depth bounds to (different) information theoretic setting,  connecting our query complexity measure to a central open problem in complexity theory.

Finally, in an additional section we consider a version of debate with
randomised verifier. We show that in the interesting regime,
randomization does not help.

\textbf{Organisation.} Section 2 defines the debate model and
DQC. Section 3 contains the upper bound results  via decision trees
and circuit simulation. Section 4 establishes the PSPACE/poly
characterisation. Section 5 contains a definition of debate with
randomized verifier, and proves a lower bound of $\Omega (\log n)$
queries for functions that depend on all their variables. Section 6 discusses open problems.

\section{Definitions}\label{sec:def}

We study a {\em deterministic} interactive proof model with two competing all-powerful provers and a query-bounded verifier, $V$. Our model is inspired by debate protocols for AI safety \cite{irving2018ai}, but analysed through the lens of query complexity. As is typical in prover-verifier settings, the verifier does not trust either prover, instead relying on the adverserial nature to achieve a reliable result.

Let  $f:\{0,1\}^n \to \{0,1\}$ be a fixed  Boolean function, known to both provers and to the verifier. For an input $x \in \{0,1\}^n$, known to the verifiers but {\em not} to the prover,
\begin{itemize}
\item \textbf{Prover 0:} Tries to convince verifier $M$ that $f(x)=0$
\item \textbf{Prover 1:} Tries to convince verifier $M$ that $f(x)=1$
\end{itemize}

For some $k$, fixed in advance, 
The provers alternately write messages $\alpha_1, \beta_1, \ldots, \alpha_k, \beta_k$, where Prover~0 writes $\alpha_i$ and Prover~1 writes $\beta_i$. We denote $\alpha = (\alpha_1, \ldots, \alpha_k)$ and $\beta = (\beta_1, \ldots, \beta_k)$, and refer to $T = (\alpha, \beta)$ as the \emph{debate transcript}.

The verifier $M_\ell$ is a deterministic query machine that adaptively queries $\ell$ bits from the concatenated string
$
x_1,\ldots,x_n,\alpha_1, \beta_1, \ldots, \alpha_k, \beta_k
$
and outputs $M(x,\alpha,\beta) \in \{0,1\}$.

We say that the debate protocol is a $(k,\ell)$-debete for $f$  if the honest prover can always force the correct outcome against any strategy by the dishonest prover. 

{\bf Formally: } A $(k,\ell)$ debate for $f$ is a set of two deterministic strategies of length $k$ each, one for a honest Prover~0, and one for a honest Prover~1, and  a deterministic $\ell$-query machine $M_\ell$ on $2k+n$ inputs (that is the verifier), such that,  
\begin{equation}\label{1case}
\mbox{For every } ~ x \in f^{-1}(1)~ ~  \forall \alpha_1 \exists \beta_1 \; \ldots \; \forall \alpha_k \exists \beta_k: \;\; M_{\ell}(x,\alpha,\beta)=1  \qquad \text{(Prover 1 wins)}
\end{equation}
and
\begin{equation}\label{0case}
\mbox{For every } ~ x \in f^{-1}(0) ~ ~ \exists \alpha_1 \forall \beta_1 \; \ldots \; \exists \alpha_k \forall \beta_k: \;\; M_{\ell}(x,\alpha,\beta)=0  \qquad \text{(Prover 0 wins)}
\end{equation}

The alternating quantifiers capture the game-theoretic nature of debate: the honest prover can respond to any move by the adversary. By De Morgan's Laws for quantifiers, conditions (\ref{1case}) and (\ref{0case}) are complementary for any $M$, ensuring the protocol is well-defined.

{\bf Comment:} The length of debate, $k$, is fixed (independent of $x$) and is known to the parties.

\begin{definition}[Debate Query Complexity]\label{def:DQC}
The \emph{Debate Query Complexity} of a Boolean function $f:\{0,1\}^{n} \rightarrow \{0,1\}$, denoted $DQC(f)$, is the smallest $\ell$ for which there exists a $(k,\ell)$-debate protocol for $f$.
\end{definition}
{\bf Note: } In the definition of $DQC(f)$ we do not have any bound on $k$. Namely, $DQC(f)$ measures only the number of bits queried by the verifier, not the total length of the transcript, which may be substantially larger.

\begin{definition}[Valid transcript, valid debate system]\label{def:valid-transcript}

For a verifier $M$ and provers $A$, $B$, as above, we say $(A, B, M)$ is a \emph{valid debate system} for $f$ on input $x$ if it correctly outputs $f(x)$.

A transcript $T=(\alpha, \beta)$ is a \emph{valid transcript} for
$f$ on input $x$ if $T$ is obtained as the interaction between an
honest prover and a possibly dishonest prover on $x$. A transcript $T=(\alpha, \beta)$ is a \emph{valid transcript} for
$f$ if there exists an input $x$ for which $T$ is valid transcript for
$f$ on $x$. 
\end{definition}
\medskip

\subsection{Some basic features of a debate for f}\label{rem:remarks}
  
\begin{enumerate}
    \item\label{item:1} As in PCP theory or the definition of
      NP, we often attribute ``meaning" to debate transcripts (e.g.,
      ``$\alpha$ encodes a circuit evaluation"). However, $\alpha$ and
      $\beta$ can in principle be arbitrary strings. Any intended
      protocol structure is enforced entirely by the verifier $V$, and
      the conditions in Equations \ref{0case} and \ref{1case} above. Thus the definition of $DQC(f)$ depends only on the existence of an appropriate query machine $M_\ell$ that satisfies Equations \ref{1case} and \ref{0case}.

    Further, although the behaviour of Prover~0 and Prover~1 may
    depend on $x$ (which they both see), as $M_\ell$ is an
    deterministic $\ell$-query machine, it computes a fixed function
    $V:\{0,1\}^{2k+n}$, on the corresponding variables $\{x_i|~ i \in
    [n]\} \cup \{\alpha_o,\beta_i| ~ i \in k\}.$ In particular, $V$
    may be different from $f$.

 \item As explained above, the function $V$ in the previous item completely determines $f$: one can compute (i.e., write a formula for) $f$ from $V$ by replacing the quantified variables with $\wedge$- and $\vee$-gates in the standard way. It follows immediately that if $f$ depends on all its variables, then $V$ must also depend on every variable in $\{x_1,\dots,x_n\}$ (though it need not depend on all of the additional $2k$ variables).

\item As $V$ is a function whose deterministic query complexity is at most $\ell$, it depends on at most $2^\ell$ variables.  This implies the following:
\begin{itemize}\label{item:3}
    \item 
By the previous item, for a function $f$ that depends on all its
variables, $V$ must depend on at least $n$ variables (that is, $\{x_1,
\ldots ,x_n\}$). Hence we conclude in Corollary~\ref{obs:lower_bound},  the lower bound $\ell \geq \log n$ on the query-complexity for such $f$.
\item There is a fixed set of variables
  $S \subseteq \{x_i|~ i \in [n]\} \cup \{\alpha_i,\beta_i| ~ i \in
  [k]\}, ~ ~ |S| \leq 2^{\ell}$, so that $V$ depends only on $S$. This
  may suggest that the transcript length $2k$ may be bounded by
  $2k \leq 2^{2\ell}$, as we may omit $\alpha_i,\beta_i$ from the
  transcript if neither are queried. While this is true, as shown in
  Lemma~\ref{lem:up_transcript}, it is not directly obvious. Since the
  strategies are adaptive, apriori, it could be that in order to output
  $\alpha_{i+1}$ that is queried, say for the honest Prover~0, the
  prover must see what was
  $\beta_i$, although $\alpha_i,\beta_i$ are not going to be queried.
\end{itemize}
\item We have defined a debate starting with Prover~0. A similar definition could be used when Prover~1 starts, however, the difference is by at most $1$ query. Similarly, in our definition of debate we  force alternations in the transcript.  This could be avoided if for all $x$, there are longer messages, but which might save on transcript length at most a factor of $2$.

\item Our definition uses deterministic verifiers, but, in many settings, randomised verifiers are often more efficent. We show in a later section that the improvements for randomised DQC are modest, and for the primarily interesting case of $\mathsf{PSPACE}$, provide no improvement.

\end{enumerate}

\begin{lemma}\label{lem:up_transcript}
If (A,B,M) is a valid $(k,\ell)$-debate system  for $f$, then there is a
(possibly different) valid $(k',\ell)$-debate for $f$ in which $k \leq 2^{\ell}$.  
\end{lemma}
\begin{proof}
  Let $V$ be the boolean function that is computed by the verifier
  machine $M$ (see item~\ref{item:1} above). Let $S$ be the set of
  all variables that $M$ possibly queries when computing $V$ on every
  possible input $x$ and transcript.  As remarked in
  item~\ref{item:3} above, $S$ is a fixed set of size $|S| \leq
  2^{\ell}$.

  Suppose that $\alpha_i, \beta_i \notin S$ for a certain $i \in
  k$. We will show that a new valid $(k-1,\ell)$-debate for $f$ can be constructed
  by deleting the $i$th round of the debate, keeping the verifier, and
  hence the function $V$ unchanged.  Seemingly, this is immediate:
  from the view point of the verifier, since its function $V$ does not
  depend on $\alpha_i,\beta_i$ they can be eliminated from the
  transcript, when viewed as the partial list (string) of input variables. However,
  just eliminating the $i$th round possibly makes the transcript
  invalid as the bits in the $(i+1)$th round depend on the valued
  exchanged at the $i$th round. 
     We show that despite this adaptive nature of the debate game,
     such rounds can be eliminated.

  Let $i$ be the smallest index for which both $\alpha_i, \beta_i \notin
  S$.  The idea is that the strategies for the new Prover~0, Prover~1,
  new-Alice, new-Bob respectively, will be identical to that of Alice
  and Bob for the first $i-1$ rounds of the debate. Then at round $i$,
  The honest prover will assume an arbitrary response for the
  $i$-round, of the `cheating' one, while simulating its own response,
  and thus being able to `jump' right away to the $i+1$ round.

  Formally, for an input $x \in f^{-1}(0)$, new-Alice having done the
  first $i-1$ rounds have full information of what her $\alpha_i$ bit
  be. She then assumes that $\beta_i=0$ (this is an arbitrary decision,
  but it is consistent, as a cheating Bob might have responded that
  way). Then, at that point, new-Alice knows how to resume the debate
  at round $i+1$ skipping doing the $i$- round.  She can keep
  following her strategy where she should be able to make a win.

  Similarly, for $x \in f^{-1}(1)$, new-Bob assumes that $\alpha_i=0$
  (again, an arbitrary decision), and hence is able to compute is own
  $\beta_i$ without actually outputing it. Then the possibly cheating
  A' can resume its part outputing $\alpha_{i+1}'$ which has no
  specific importance (as possibly being a cheat), on which new-Bob
  will be able to simulate Bob's $(i+1)$-move, making it the $i$th
  move of the new transcript.

  We note that in the new debate, we have a more restricted set of
  strategies (as we eliminating a round), but since new-Alice can
  force a win on $0$-input, while new-Bob can force a win on
  $1$-inputs the new system is valid for $f$ (we used here that $V$
  does not query $\alpha_i, \beta_i$ on any possible run, hence from
  his point of view, the transcript bits it sees could have been in
  the original debate).

 Hence, skipping rounds in this way can be done as long as there is
 still an index $i$ as above. Since $|S| \leq 2^{\ell}$ this can be
 done as long as $k > 2^{\ell}$ which completes the proof. 
\end{proof}

\begin{corollary}\label{obs:lower_bound}
For any Boolean function $f:\{0,1\}^{n} \rightarrow \{0,1\}$ that depends on all of its variables,
\[DQC(f) \geq \log n.\]
\end{corollary}

\section{Circuit Upper Bounds}

We now establish upper bounds for DQC. We obtain two types of bounds: a depth-based bound $DQC(f) \leq depth(C_f)$ using an adaptation of Karchmer-Wigderson games, and a size-based bound $DQC(f) \leq \log(size(C_f)) + 3$ using cross-examination. These bounds employ different techniques with different tradeoffs, and are useful for different classes of circuits.

We use $C_f$ to refer to the fan-in-2 AND/OR circuit implementing function $f$. Without loss of generality we assume that $C_f$ is leveled with a top AND gate and all gates in subsequent levels alternate between AND and OR gates, although this is only necessary for subsection \ref{sect:depth upper bound}. We use $depth(C_f)$ for the depth of $C_f$ and $size(C_f)$ for the number of gates. We may also presume that the size/depth are powers of 2 to avoid using ceiling operators. 

\subsection{Depth Upper Bound}\label{sect:depth upper bound}

Our first result is based on the Karchmer-Wigderson game. Karchmer and Wigderson noticed a relation between the circuit depth of a Boolean function $f$ and the communication complexity of the following relation: Alice is given $x \in f^{-1}(0)$ and Bob is given $y \in f^{-1}(1)$, and they must communicate to output an index $i$ such that $x_i \neq y_i$. The minimal communication required equals the fan-in-2 $(\wedge,\vee)$-circuit depth of $f$. We adapt this game to the debate setting, where provers traverse the circuit structure via alternating messages to convince the verifier of the output value.

\begin{theorem}\label{thm:depth-upper-bound}
Let $C_f$ be a circuit for $f$ over any basis. Then
\[
DQC(f) \leq depth(C_f) + 1
\]
\end{theorem}

\begin{proof}
We assume w.l.o.g.\ that the top gate of the circuit $C_f$ is an $\wedge$-gate (AND-gate). Our proof makes use of the game constructed by Karchmer and Wigderson for their communication complexity setting.

Let Prover~0 assume the role of the \emph{defender}, and Prover~1 act as the \emph{challenger}. Specifically, if $f(x) = 0$, then Prover~0 can successfully defend the claim that $f(x) = 0$, forcing the verifier to output $0$. Conversely, for inputs $x$ such that $f(x) = 1$, Prover~0 must inevitably fail, and under optimal play by Prover~1, the verifier will output $1$.

The game begins at the top gate of the circuit, corresponding to the
output bit. This final gate is an $\wedge$, thus if indeed $f(x)=0$ as
Prover~0 claims, the output of this gate is a 0,  and hence one of the inputs must have been a $0$. The first move (bit = $\alpha_1$) of Prover~0 is to point to this input as one of the two input legs. We trace this input leg to the next gate, which is an $\vee$. As Prover~0 claims this gate evaluates to $0$, if Prover~1 wants to challenge this, he should be able to point to input of that gate that is $1$. So, this would be the interpretation of its bit $\beta_1$. 

This interaction continues iteratively: the provers descend through the circuit, gate by gate, toward the input level. In its turn - Prover~0 will always be at a $\wedge$-gate he claims is $0$, by pointing to a $0$ input to it, and Prover~1 will be at a $\vee$-gate claiming that it is $1$ by pointing to a $1$-input to it.  Eventually, the process reaches a circuit input, which can be directly compared to the relevant input $x_i$ by the verifier.

At every stage of this game each honest prover should have been able to point to their preferred value if the final output was as they claimed. The verifier can read the $depth(C_f)$ steps (that is - the whole transcript), to determine which input the game points to, then use this value as the final output of the debate.
\end{proof}

The depth bound is optimal for shallow circuits such as those in AC$^0$ or NC$^1$. However, it can be unsatisfying for functions with deep but small circuits. For instance, a circuit of size $n^2$ but depth $n$ would give a DQC upper bound of $n$, while we might hope for something closer to $\log(n^2) = 2\log(n)$. We address this limitation with our next result.

\subsection{Size Upper Bound via Cross-Examination}\label{sec:ub_circuit}

We can achieve logarithmic dependence on circuit size by employing a
different technique: \emph{cross-examination}. In this paradigm, one
prover performs an entire computation and writes down all intermediate
steps, while the other prover challenges  the computation by  pointing
out  a single location where an error occurred. The verifier need only check the location identified by the second prover. This strategy trades debate length (now linear in the computation size) for query efficiency (now logarithmic), as the verifier requires only $O(\log m)$ bits to specify a location in a computation of size $m$.

\begin{theorem}\label{thm:size-upper-bound}
    $DQC(f)\leq \log(size(C_f))+3$
\end{theorem}

\begin{proof}
    Prover~0, whom we now call Alice, makes the first $size(C_f)$ moves, outputting the value of every gate in the circuit. Prover~1, whom we call Bob, then either accepts Alice's computation or points to the location of a gate whose output does not follow from its two inputs (which may be circuit inputs or outputs of other gates that Alice has claimed).

    The verifier reads Bob's output, which specifies a gate index (requiring $\log(size(C_f))$ queries). The verifier goes to that location in Alice's transcript and queries Alice's claimed value for that gate (1 query) and for its two input wires (2 queries). Since the circuit structure is known to all parties, the verifier can determine which gate type this is and check whether Alice's output matches the logic table for that gate type applied to the two inputs.
    
    If Alice's gate has been incorrectly implemented, the verifier outputs 1. Additionally, if Bob has pointed to the output gate and its value is 1, the verifier outputs 1. Otherwise, the verifier outputs 0.

    The total number of queries is $\log(size(C_f)) + 3$.
\end{proof}

This cross-examination technique is quite general and will prove even more powerful in the next section, where we use it to compress arbitrary polynomial-time verifiers into logarithmic query complexity.

Our circuit upper bounds have an intriguing consequence for circuit complexity: they imply that proving strong lower bounds on DQC is at least as hard as proving circuit lower bounds. Since the best known explicit circuit lower bound for languages in $\mathsf{NP}$ is approximately $5n$ \cite{iwama2002explicit}, our size upper bound creates a barrier to improving DQC lower bounds beyond logarithmic.

\begin{corollary}[Connection to circuit lower bounds]\label{cor:circuit-lower-bounds}
Any sufficient lower bound on $DQC$ for a function in $\mathsf{NP}$ would yield new circuit lower bounds. Specifically, if there exists a language $L \in \mathsf{NP}$ with characteristic function $f$ such that $DQC(f) \geq \log n + 6$, then $L$ requires circuits of size at least $8n$.
\end{corollary}

\begin{proof}
By Theorem~\ref{thm:size-upper-bound}, any Boolean function $f$ with a circuit of size $m$ satisfies $DQC(f) \leq \log(m) + 3$. 

Let $L \in \mathsf{NP}$ be a language with characteristic function
$f:\{0,1\}^n \to \{0,1\}$, and let $\{C_n\}$ be a family of the
smallest circuits computing $f$ (in terms of their size). Suppose $DQC(f) \geq \log n + 6$. Then:
\[
\log n + 6 \leq DQC(f) \leq \log(\text{size}(C_n)) + 3
\]

Rearranging gives $\log(\text{size}(C_n)) \geq \log n + 3$, which implies $\text{size}(C_n) \geq 2^{\log n + 3} = 8n$.

Since the best known circuit lower bound for explicit languages in $\mathsf{P}$ is approximately $5n$ \cite{iwama2002explicit}, proving $DQC(f) \geq \log n + 6$ for any such language would constitute a breakthrough in circuit complexity.
\end{proof}

\section{Cross-Examination and PSPACE}

A key tool in the previous section was cross-examination: having one prover perform a computation while the other prover either accepts or points to an error. By formalising cross examination we are able to prove that $\mathsf{PSPACE/poly}$ is precisely the class of functions decidable with $O(\log n)$ queries. The proof proceeds in two steps: first, we establish a general cross-examination lemma showing that any polynomial-time verifier can be simulated with logarithmic queries. Second, we use this lemma to prove both directions of the characterisation.

\subsection{The Cross-Examination Lemma}

The power of cross-examination lies in its ability to compress verification: rather than having the verifier execute a polynomial-time computation, we have one prover simulate that computation while the other verifies a single step. This reduces the query complexity from polynomial to logarithmic.

\begin{lemma}[Cross-Examination Lemma]\label{lma:cross-examination}
Suppose that $f$ has a $(k,\ell)$-debate in which the deterministic
verifier  computes a function $V$ (on the input bits $(\alpha_1,
\beta_1, \ldots \alpha_k, \beta_k, x_1, \ldots x_n)$ that is
computable by a circuit  $C$ of size $m$. Then
\[
DQC(f) \leq \log(m) + 3
\]
\end{lemma}

\noindent\textbf{Remarks:} The verifier computes a boolean function
$V$ that has, by definition,  query complexity $\ell$. Hence $V$ has a
circuit of size at most $2^{\ell}$. However, it could be that the
optimal size for $V$ is $m << 2^{\ell}$. The lemma then uses the
result in Section~\ref{sec:ub_circuit}, to form (possibly another)
debate for $f$ of query complexity $\log( m) +3,$ by verifying the
computation of $V$ rather computing it. Namely, it uses another
verifier (function) $V'$ (and also augmented provers), to verify the
function $V$ in addition to the verification of the value of $f$. 

A second remark is that, this lemma {\em cannot} be recursively
applied to its own verifier to get an ever decreasing  a better
complexity, since when $m \geq 2^{\ell}$ (that is, the decision tree of
$V$ is balanced, the complexity actually deteriorates (by the additive
$3$-term)). 

\begin{proof}
Let Alice and Bob (Prover~0, Prover~1, respectively) perform their
debate of $k$ rounds,  as per the original procedure, producing
transcript $T = (\alpha, \beta)$. In our new debate, new-Alice makes
additional $m$ moves, outputting the value of every gate in the
verification circuit $C$ that computes $V$ on the variables of $\alpha,
\beta, x$. Either new-Bob points to the claimed output value $1$
(concluding that $V$ should have  accepted contrary to what Alice goal
was), leading new-V to accept, {\em or} new-Bob points to the location of a gate whose output does not follow from its inputs.

The new-Verifier $V$, reads Bob's output (requiring $\log(m)$ queries) which specifies a gate index. It then queries Alice's claimed values
for that gate and its two inputs (3 additional queries). Since the circuit
structure is known to all parties, the verifier can determine the gate
type and check whether Alice's output matches the logic table. If the gate Bob points to is an input to the circuit, then the verifier queries the corresponding location in the original transcript. 

If Alice's gate has been incorrectly implemented, the verifier is
convinced by new-Bob's claim that indeed Alice was wrong. Otherwise, it
conclude that new-Alice, simulating Alice is correct and decides for Alice.

Correctness follows from the same argument as
Theorem~\ref{thm:size-upper-bound}: for a $0$-input to $f$, if Alice
is honest about the circuit evaluation, new-Bob cannot point to an
error, and the new-V reading the output value would decide for Alice,
as it should. For a $1$ input to $f$,  the original $V$ must compute
$1$, hence Bob should be able to point to an erroneous gate (or that
the output gate computes $1$), making new-$V$ to vote for Bob
regardless to a possibly cheating new-Alice.
\end{proof}

This lemma is powerful because it applies to \emph{any}
polynomial-time verifier: if a function can be computed by a debate
verified by a polynomial-size circuit, then it has logarithmic DQC,
although possibly, the query-complexity of this function might be
as large as the input size $n$.
We now use this to characterise the class of functions with $O(\log n)$ query complexity.

\subsection{Characterising PSPACE/poly}

\begin{theorem}\label{thm:pspace-characterisation}
$$\mathsf{PSPACE/poly} = \{f : DQC(f) \leq O(\log n)\}$$
\end{theorem}

The proof proceeds by establishing both inclusions separately. The
forward direction uses the cross-examination lemma: since PSPACE has
polynomial-time verifiers, we can compress them to logarithmic
queries. In the reverse direction we show that logarithmic query
complexity implies that the debate can be compressed to polynomial length
with polynomial-time verification using advice.

\begin{lemma}\label{lem:pspace-upper}
$\mathsf{PSPACE/poly} \subseteq \{f : DQC(f) \leq O(\log n)\}$
\end{lemma}

\begin{proof}
$\mathsf{PSPACE}$ can be decided by debates of polynomial length with
polynomial-time verifiers\footnote{A sketch for completeness: If $L$
  in PSPACE, it means that there is a polys-pace machine, that for
  every input (of size $n$) and a state $w$ (encoded by $poly(n)$ string),
  give (in linear time in $|w|$) the next state. Thus a verification
  of acceptance can be done by 'Savich-like' procedure verifying the
  computation path (that can be $exp(n)$ length), but reducing at each
  query to this machine, the path-length by a factor of $2$ (using the
  challenger Prover~0) to point to the allegedly wrong half of the
  claimed computation path). Thus
  after $poly(n)$ time we reach to adjacent states which can be
  verified by the verifier.}  \cite{irving2018ai}. This extends to the non-uniform case, i.e. $\mathsf{PSPACE/poly}$ can be decided by debates of polynomial length judged by polynomial-size circuits.

By the Cross-Examination Lemma (Lemma~\ref{lma:cross-examination}), this debate can be verified with
\[
DQC(f) \leq \log(\text{size}(C_n)) + 3 = \log(\text{poly}(n)) + 3 = O(\log n)
\]
queries. Therefore $\mathsf{PSPACE/poly} \subseteq \{f : DQC(f) \leq O(\log n)\}$.
\end{proof}

\begin{lemma}\label{lem:pspace-lower}
$\{f : DQC(f) \leq O(\log n)\} \subseteq \mathsf{PSPACE/poly}$
\end{lemma}

\begin{proof}
    Let $\{f_n: \{0,1\}^n \to \{0,1\},~  n \in \mathbb{N} \}$ be a set of functions for which $DQC(f_n) =\ell =O( (\log(n))$. To prove the theorem we will take a language from this set, $L=\{x\in f_n^{-1}(1)~ | ~ ~ n \in \mathbb{N} \}$, and place it in $\mathsf{PSPACE/poly}$ by describing a polynomial time verifier acting on a polynomial time debate taking polynomial length advice (which is the class $\mathsf{PSPACE/poly}$ \cite{irving2018ai}).

    Fix any constant $a$ such that the DQC of the functions in $L$ are upper bounded by $a\log(n)$.
    By Lemma~\ref{lem:up_transcript} the debate length, $l$, is now upper bounded by $2^k=n^a$, a polynomial. 

    At each query step the verifier makes one adaptive query, therefore in $a \log(n)$ queries (depending on the answers it recieves) the verifier will only query $2^{a \log(n)}-1<n^a$ different locations. 
    We can completely describe the verifier by giving a map from query response strings (strings of the bit outputs of each query) to next query location, and, finally, from $a\log(n)$-length query responses to verifier outputs.
    This can be described by a look-up table of $n^{a+1}$ rows with $2a\log(n)$ length elements. This look-up table is therefore polynomial length. 

    Our debate (in the sense of $\mathsf{DEBATE}$ \cite{irving2018ai}) protocol is now, hopefully, apparent: The verifier-look-up table is provided as advice, the two debaters simulate the debate used in $DQC$ and the verifier uses the lookup table to evaluate it. 
\end{proof}

This completes the proof of Theorem~\ref{thm:pspace-characterisation}.


\noindent\textbf{Remark on computational efficiency.} The verifier constructed in Lemma~\ref{lem:pspace-upper} is not only query-efficient but also computationally simple (depending on the model of query access): it simply has to parrot the location given by Bob to request it as a query string, modify that string to give two other simple locations, and then perform a 2 input logic gate to check the computation. 
If we were dealing with e.g. Turing machines which have local structure, the verifier could be a constant depth circuit.

\noindent\textbf{Implications for scalable oversight.} Theorem~\ref{thm:pspace-characterisation} shows that all debate-decidable languages can be verified with $O(\log n)$ queries using computationally efficient verifiers. This provides strong theoretical support for debate as a practical scalable oversight mechanism: a human judge needs to examine few locations in the debate transcript, and the verification logic at each step can be extremely simple. The human cost thus scales with $\log(n)$ rather than with the computational complexity of the task being verified.

\section{Lower bound on randomised verifiers}\label{sec:randomised}
We have seen in Corollary~\ref{obs:lower_bound} that $DQC(f) \geq
\log n$, for functions that depend on all their variables.  In
general, functions that have deterministic query complexity $D$ might
have randomised complexity, and even $0$-error, randomised decision
tree of depth $o(D)$. Thus, one could ask whether a randomised
verifier (viewed as a randomised decision tree) can be significantly
more efficient than  the best deterministic verifier.  We show that this is not the case for the
interesting regime of $\mathsf{PSPACE/poly}$ functions.

We first need to define what we mean by a randomised verifier:
\begin{definition}[randomised verifier]
  Let $(A,B,V)$ be a valid $(k,\ell)$-debate system for $f$, as per
  Section~\ref{sec:def}. That is $V:\{0,1\}^{2k+n} \to \{0,1\}$,
  for which Equations \ref{1case}, \ref{0case} hold. Let $T_R$ be a {\bf randomised} decision tree computing the function $V$  making $q$ queries.  We say that $T_R$ has error bounded by $1/3$ if for every $x$ and valid transcript $T_x$, it correctly computes $V(T_x,x)$ with probability at least $2/3$.
\end{definition}

\textbf{Remarks on definition:} 
\begin{itemize}
    \item We do not assume here any bound on $k$ w.r.t.\ $\ell$, or any restriction of the computational resources of the verifier. We simply require the existence of the function $V$ (and corresponding strategies for Prover~1, Prover~0) as per Equations \ref{0case} and
  \ref{1case}.
    \item The provers are still deterministic (as they are all powerful). Moreover, it is not quite clear how to define other variants of randomised debates so they  would still be convincing and useful.  
\end{itemize}

\begin{theorem}\label{thm:lb_randomised}
  Let $f: \{0,1\}^n \to \{0,1\}$ depend on all its variables, then
  any randomised verifier achieving error at most $1/3$ must query
  at least $\log (n)-3$ queries (in the worst case).
\end{theorem}
\begin{proof}
  Let $A,B,V$ be a valid debate system for $f$, where $A,B$ are
  Prover~0, Prover~1, as before, and $V: \{0,1\}^{2k+ n} \to \{0,1\}$
  is the 'game' verification function as defined by Equations
  \ref{1case} and \ref{0case}.

  Let $M$ be a $q$-query randomised machine that computes $V$ in $q$
  queries. We claim that if $q < \log n -3$ then the error $M$ makes
  is more than $1/3$. To prove this we use Yao's principle \cite{yao1977probabilistic},
  namely we define a distribution $D$ on valid inputs (to $V$, not to $f$)
  and show that any {\em deterministic} $q$-query decision tree
  makes an error larger 
  than $1/3$ w.r.t $D$.

  Recall that since $f$ depends on all its variables, this means that
  for every position $i \in [n]$, there exist inputs
  $x, x' \in \{0,1\}^n$ that differ only at position $i$ such that
  $f(x) \neq f(x')$.  We fix for every $i$ such a pair of inputs
  denoted $w_i,\tilde{w}_i$, with $f(w_i)=0, ~ f(\tilde{w}_i)=1$. Further, we assume, for
  simplicity that these pairs are disjoint for different $i$'s
  (otherwise, we will give weights  to each input according
  to its multiplicity). Thus we assume that all together we have
 two sets $W_0 = \{w_i, ~ i \in n \}$ and $W_1 = \{\tilde{w}_i, ~ i
 \in n \}$, both of size $n$.

 For each $i\in [n]$, we construct a debate transcript $T(i)$
 that is valid, simultaneously,  for the two inputs $w_i,\tilde{w}_i$ that differ only at
 position $i$. 
 
We construct $T(i)$ inductively by interleaving honest and dishonest
prover responses, as follows.

On input $w_i$, the honest Prover~0 (who claims $f(w_i) = 0$) would
send some first message $\alpha_1$. Let the dishonest Prover~1 respond
with the message $\beta_1$ that Prover~1 would honestly send on input $\tilde{w}_i$
after seeing $\alpha_1$. We then continue this process: Prover~0, on
input $w_i$ 
responds honestly after  seeing $\alpha_1, \beta_1$, producing
$\alpha_2$; Then Prover~1 responds as it would on $\tilde{w}_i$ seeing
$\alpha_1, \beta_1, \alpha_2$, producing $\beta_2$; and so forth.

More formally:
\begin{itemize}
\item $\alpha_1$ is the honest Prover~0 response on input $w_i$ after seeing no prior messages
\item $\beta_j$ is the honest Prover~1 response on input $\tilde{w}_i$
  after seeing the prefix $\alpha_1, \ldots, \alpha_j$
\item $\alpha_{j+1}$ is the honest Prover~0 response on input $w_{i}$
  after seeing prefix $\alpha_1, \beta_1, \ldots, \alpha_j, \beta_j$
\end{itemize}

This produces transcript $T(i) = (\alpha, \beta)$, which, by
construction, is valid for $w_i$ because Prover~0 (the honest prover
for $w_i$) follows its honest strategy while Prover~1 deviates
arbitrarily.  To see that the same $T(i)$ is valid for $\tilde{w}_i$,
reverse the roles. On input $\tilde{w}_i$, the honest Prover~1 would
produce exactly the sequence $\beta$ (by our construction), while
Prover~0 is forced to produce $\alpha$ by construction. Thus $T(i)$ is also valid for
$\tilde{w}_i$.

We now can define the distribution $D$ - it will be uniform on the set
$W_0 \cup W_1$ of size $2n$, where for each $w_i,\tilde{w}_i$ we
append transcript $T(i)$.

Let $M_q$ be any deterministic decision tree, making at most $q$
queries on the input $(\alpha, \beta, x)$ and computes $V$. We claim
that if $q < \log n -3$ than $M_q$ errors with probability larger than
$1/3$ on input chosen according to the distribution $D$.

We first note
that if $M_q$ does not query $x_i$ on input $(T(i),w_i)$, it does not
query $x_i$ also on $(T(i),\tilde{w}_i)$. Hence, in particular in
this case, both above inputs will reach the same leaf of the decision
tree $M_q$. We conclude that the answer to one of these two inputs will be erroneous.

Now, as $q < \log -3$, the deterministic decision tree $M_q$ for $V$
contains altogether less than $n/8$ queries (in all possible
runs for all possible inputs), and in particular it contains less
than $n/8$ of the $x_i$ as queries. We conclude that for at least
$7n/8$ of the $i \in [n]$ the pair of inputs  $((T(i) w_i), ~
(T(i)\tilde{w}_i )$ arrive both to the same leaf of $M_q$. Hence
$M_q$ makes error on at least $7n/16$ of the inputs supported by $D$,
resulting in an error at least $\frac{7}{16} > \frac{1}{3}$.
\end{proof}

This thoerem shows that the $O(\log(n))$ bound for $\mathsf{PSPACE}$ cannot be meaningfully improved as (e.g.) the parity function cannot be improved beyond $\log(n)-3$ randomised query complexity and is certainly in $\mathsf{PSPACE}$.

Finally, we might ask the deterministic query complexity of a function of a given randomised query complexity. 

\begin{theorem}
    Let $(A,B,V)$ be a valid $(k,\ell)$-debate system for $f$, suppose that the randomised query complexity of $f$ is $q$, then the deterministic DQC is
    $$
    DQC(f) < q + \log(2k+n)+ 10
    $$
\end{theorem}
Informally, this theorem is proven by finding a set of deterministic decision trees which correctly implement $V$ for all points when a majority vote is taken. We then turn those deterministic trees into circuits and take a majority vote to get a debate protocol with a deterministic verifier which gives the result by cross examination.

\begin{proof}
    By Newman's theorem \cite{newman1991private}, we can deterministically implement the random decision tree verifier with a majority vote on  $12(2k+n)$ deterministic trees, as the input is at most $2k+n$ long. Each of these trees is a branching binary tree of depth $q$, so can be implemented by a $3\cdot2^q$-gate circuit and majority vote on $12(2k+n)$ variables, adding another $60(2k+n)$ gates, and giving a deterministic verifier circuit with $12(2k+n)(3\cdot2^q+5)$ gates.

    By Lemma \ref{lma:cross-examination} (cross-examination) a debate that can be verified by a circuit of size $m$ has $DQC\leq \log(m)+3$, therefore (assuming $q\geq1$):
$$
        DQC(f)\leq \log(12(2k+n)(3\cdot2^q+5))+3 
$$
$$  
        \leq q  + \log(2k+n) + 3 + \log(66)
$$
\end{proof}

\section{Conclusion}

We have introduced Debate Query Complexity (DQC) as a measure of the human oversight cost in debate protocols. Our results establish that functions depending on all variables require $\Omega(\log n)$ queries, while circuit-based upper bounds show $DQC(f) \leq O(\text{depth}(C_f))$ and $DQC(f) \leq O(\log(\text{size}(C_f)))$. Our main result characterises $\mathsf{PSPACE/poly}$ as precisely the class of functions decidable with $O(\log n)$ queries. 

This logarithmic bound has concrete practical implications: for inputs of size one million bits, verification requires only around 20 queries regardless of the computational complexity of the problem being decided. Moreover, the verification circuit can be extremely simple. This provides strong theoretical support for debate as a scalable oversight mechanism.

Our techniques (adapting Karchmer-Wigderson games and formalising cross-examination) connect debate to established areas of complexity theory, particularly communication complexity and circuit complexity. This suggests DQC may serve as a useful lens for studying both AI alignment and fundamental questions in complexity.

\subsection*{Open Problems}

\textbf{Computational Realism and Practical Alignment}
\begin{itemize}
    \item \textbf{Empirical validation}: Do debates between actual language models exhibit the logarithmic query scaling our theory predicts? Which debate protocols from our theoretical analysis transfer effectively to practice?

    \item \textbf{Computationally bounded debaters}: Our results assume unbounded debaters. Brown-Cohen et al. \cite{brown2023scalable, brown2025avoiding} study debates where debaters are polynomially bounded and may not know the verifier's code. How do query requirements change under these realistic constraints? Can logarithmic query complexity be preserved when debaters face computational limits?

    \item \textbf{Verifier opacity}: When the verifier is a large language model or other opaque system, debaters cannot exploit its structure. Does this change the query complexity? Can we characterise which verification tasks remain efficient under verifier opacity?

    \item \textbf{Human value oracles}: Recent work models human judgment as oracle access rather than deterministic computation. How does incorporating human value queries into the protocol affect the overall query budget?
\end{itemize}

\textbf{Complexity Theory}
\begin{itemize}
    \item \textbf{The circuit lower bound barrier}: Corollary~\ref{cor:circuit-lower-bounds} shows that proving $DQC(f) \geq \log n + 6$ for any language in $\mathsf{P}$ would yield new circuit lower bounds, exceeding the best known explicit bounds. This suggests two possibilities: either all efficiently computable functions have extremely low $DQC$, or a circuit lower bound better than the current state of the art can be proved using a lower bound for DQC. Of course, in the later case it could be that proving tight DQC lower bounds faces the same barriers as circuit lower bounds. Understanding which scenario holds would clarify the fundamental limits of query-efficient debate. Conversely (and perhaps more excitingly) proving a relatively simple bound on $DQC$ of a single function would immediately yield a breakthrough in circuit complexity

    \item \textbf{Beyond Karchmer-Wigderson}: Our circuit upper bounds adapt KW games from communication complexity. What other communication complexity tools yield DQC bounds? Can techniques from information complexity, multiparty communication, or streaming algorithms provide new debate protocols or lower bounds?

    \item \textbf{Multiple debaters}: Does adding a third debater (or more) reduce query complexity for some function classes? The alternating quantifier structure $\exists \forall \exists \forall \ldots$ generalises naturally to more players, but it is unclear whether this provides query advantages without assuming a trivially advantageous set up.
\end{itemize}

Ultimately, our work provides a foundation for reasoning about the human cost of verification in debate protocols. The logarithmic query bound offers hope that scalable oversight is achievable in theory. The challenge now is to determine whether this theoretical efficiency can be realised in practice, and to understand the fundamental barriers, both computational and complexity-theoretic, that constrain what debate can achieve.
\printbibliography

\end{document}